\def\year{2018}\relax
\documentclass[letterpaper]{article} 
\usepackage{aaai18}  
\usepackage{times}  
\usepackage{helvet}  
\usepackage{courier}  
\usepackage{url}  
\usepackage{graphicx}  

\usepackage[utf8]{inputenc}

\usepackage[]{todonotes}

\usepackage{amssymb}
\usepackage{amsmath}
\usepackage{amsthm}

\newtheorem{theorem}{Theorem}

\theoremstyle{definition}
\newtheorem{definition}{Definition}

\usepackage[nolist,nohyperlinks]{acronym}
\usepackage[inline]{enumitem}

\usepackage{pgf}
\usepackage{tikz}
\usetikzlibrary{arrows,automata}
\usetikzlibrary{shapes,snakes}
\usetikzlibrary{intersections}

\usepackage{mathtools,algorithm}
\usepackage[noend]{algpseudocode}
\usepackage[capitalise]{cleveref}

\makeatletter

\pgfkeys{/pgf/.cd,
  rectangle corner radius/.initial=3pt
}
\newif\ifpgf@rectanglewrc@donecorner@
\def\pgf@rectanglewithroundedcorners@docorner#1#2#3#4{%
  \edef\pgf@marshal{%
    \noexpand\pgfintersectionofpaths
      {%
        \noexpand\pgfpathmoveto{\noexpand\pgfpoint{\the\pgf@xa}{\the\pgf@ya}}%
        \noexpand\pgfpathlineto{\noexpand\pgfpoint{\the\pgf@x}{\the\pgf@y}}%
      }%
      {%
        \noexpand\pgfpathmoveto{\noexpand\pgfpointadd
          {\noexpand\pgfpoint{\the\pgf@xc}{\the\pgf@yc}}%
          {\noexpand\pgfpoint{#1}{#2}}}%
        \noexpand\pgfpatharc{#3}{#4}{\cornerradius}%
      }%
    }%
  \pgf@process{\pgf@marshal\pgfpointintersectionsolution{1}}%
  \pgf@process{\pgftransforminvert\pgfpointtransformed{}}%
  \pgf@rectanglewrc@donecorner@true
}
\pgfdeclareshape{rectangle with rounded corners}
{
  \inheritsavedanchors[from=rectangle] 
  \inheritanchor[from=rectangle]{north}
  \inheritanchor[from=rectangle]{north west}
  \inheritanchor[from=rectangle]{north east}
  \inheritanchor[from=rectangle]{center}
  \inheritanchor[from=rectangle]{west}
  \inheritanchor[from=rectangle]{east}
  \inheritanchor[from=rectangle]{mid}
  \inheritanchor[from=rectangle]{mid west}
  \inheritanchor[from=rectangle]{mid east}
  \inheritanchor[from=rectangle]{base}
  \inheritanchor[from=rectangle]{base west}
  \inheritanchor[from=rectangle]{base east}
  \inheritanchor[from=rectangle]{south}
  \inheritanchor[from=rectangle]{south west}
  \inheritanchor[from=rectangle]{south east}

  \savedmacro\cornerradius{%
    \edef\cornerradius{\pgfkeysvalueof{/pgf/rectangle corner radius}}%
  }

  \backgroundpath{%
    \northeast\advance\pgf@y-\cornerradius\relax
    \pgfpathmoveto{}%
    \pgfpatharc{0}{90}{\cornerradius}%
    \northeast\pgf@ya=\pgf@y\southwest\advance\pgf@x\cornerradius\relax\pgf@y=\pgf@ya
    \pgfpathlineto{}%
    \pgfpatharc{90}{180}{\cornerradius}%
    \southwest\advance\pgf@y\cornerradius\relax
    \pgfpathlineto{}%
    \pgfpatharc{180}{270}{\cornerradius}%
    \northeast\pgf@xa=\pgf@x\advance\pgf@xa-\cornerradius\southwest\pgf@x=\pgf@xa
    \pgfpathlineto{}%
    \pgfpatharc{270}{360}{\cornerradius}%
    \northeast\advance\pgf@y-\cornerradius\relax
    \pgfpathlineto{}%
  }

  \anchor{before north east}{\northeast\advance\pgf@y-\cornerradius}
  \anchor{after north east}{\northeast\advance\pgf@x-\cornerradius}
  \anchor{before north west}{\southwest\pgf@xa=\pgf@x\advance\pgf@xa\cornerradius
    \northeast\pgf@x=\pgf@xa}
  \anchor{after north west}{\northeast\pgf@ya=\pgf@y\advance\pgf@ya-\cornerradius
    \southwest\pgf@y=\pgf@ya}
  \anchor{before south west}{\southwest\advance\pgf@y\cornerradius}
  \anchor{after south west}{\southwest\advance\pgf@x\cornerradius}
  \anchor{before south east}{\northeast\pgf@xa=\pgf@x\advance\pgf@xa-\cornerradius
    \southwest\pgf@x=\pgf@xa}
  \anchor{after south east}{\southwest\pgf@ya=\pgf@y\advance\pgf@ya\cornerradius
    \northeast\pgf@y=\pgf@ya}

  \anchorborder{%
    \pgf@xb=\pgf@x
    \pgf@yb=\pgf@y%
    \southwest%
    \pgf@xa=\pgf@x
    \pgf@ya=\pgf@y%
    \northeast%
    \advance\pgf@x by-\pgf@xa%
    \advance\pgf@y by-\pgf@ya%
    \pgf@xc=.5\pgf@x
    \pgf@yc=.5\pgf@y%
    \advance\pgf@xa by\pgf@xc
    \advance\pgf@ya by\pgf@yc%
    \edef\pgf@marshal{%
      \noexpand\pgfpointborderrectangle
      {\noexpand\pgfqpoint{\the\pgf@xb}{\the\pgf@yb}}
      {\noexpand\pgfqpoint{\the\pgf@xc}{\the\pgf@yc}}%
    }%
    \pgf@process{\pgf@marshal}%
    \advance\pgf@x by\pgf@xa%
    \advance\pgf@y by\pgf@ya%
    \pgfextract@process\borderpoint{}%
    \pgf@rectanglewrc@donecorner@false
    \southwest\pgf@xc=\pgf@x\pgf@yc=\pgf@y
    \advance\pgf@xc\cornerradius\relax\advance\pgf@yc\cornerradius\relax 
    \borderpoint
    \ifdim\pgf@x<\pgf@xc\relax\ifdim\pgf@y<\pgf@yc\relax
      \pgf@rectanglewithroundedcorners@docorner{-\cornerradius}{0pt}{180}{270}%
    \fi\fi
    \ifpgf@rectanglewrc@donecorner@\else
      \southwest\pgf@yc=\pgf@y\relax\northeast\pgf@xc=\pgf@x\relax
      \advance\pgf@xc-\cornerradius\relax\advance\pgf@yc\cornerradius\relax
      \borderpoint
      \ifdim\pgf@x>\pgf@xc\relax\ifdim\pgf@y<\pgf@yc\relax
       \pgf@rectanglewithroundedcorners@docorner{0pt}{-\cornerradius}{270}{360}%
      \fi\fi
    \fi
    \ifpgf@rectanglewrc@donecorner@\else
      \northeast\pgf@xc=\pgf@x\relax\pgf@yc=\pgf@y\relax
      \advance\pgf@xc-\cornerradius\relax\advance\pgf@yc-\cornerradius\relax
      \borderpoint
      \ifdim\pgf@x>\pgf@xc\relax\ifdim\pgf@y>\pgf@yc\relax
       \pgf@rectanglewithroundedcorners@docorner{\cornerradius}{0pt}{0}{90}%
      \fi\fi
    \fi
    \ifpgf@rectanglewrc@donecorner@\else
      \northeast\pgf@yc=\pgf@y\relax\southwest\pgf@xc=\pgf@x\relax
      \advance\pgf@xc\cornerradius\relax\advance\pgf@yc-\cornerradius\relax
      \borderpoint
      \ifdim\pgf@x<\pgf@xc\relax\ifdim\pgf@y>\pgf@yc\relax
       \pgf@rectanglewithroundedcorners@docorner{0pt}{\cornerradius}{90}{180}%
      \fi\fi
    \fi
  }
}

\makeatother

\tikzstyle{line} = [draw, -latex']
\tikzstyle{hidden} = [ellipse, draw, text centered, inner sep=1pt]
\tikzstyle{obs} = [ellipse, draw, fill=gray!60, text centered, inner sep=1pt]
\tikzstyle{rv} = [draw, ellipse, inner sep=2pt]
\tikzstyle{pf} = [draw, rectangle, fill=gray]
\tikzstyle{pc} = [rectangle with rounded corners,draw,rectangle corner radius=15pt, align=center, minimum width=22mm, font=\normalsize, inner sep=3pt]
\tikzstyle{pc2} = [rectangle with rounded corners,draw,rectangle corner radius=8pt, align=center, minimum width=6mm, font=\normalsize, inner sep=2pt]

\tikzstyle{nhidden} = [draw=none,fill=none,ellipse, text centered, inner sep=1pt]
\tikzstyle{nobs} = [draw=none,fill=none,ellipse, fill=gray!60, text centered, inner sep=1pt]
\tikzstyle{nrv} = [draw=none,fill=none,ellipse, inner sep=2pt]
\tikzstyle{npf} = [draw=none,fill=none,rectangle]

\tikzstyle{ID} = [draw, circle, font=\normalsize]
\tikzstyle{INN} = [draw, circle, inner sep=1pt, fill=black]

\newcommand\pfs[8]{
  \node[pf, #1 of=#2, node distance=#3, xshift=-1mm, yshift=1mm](#6) {};
  \node[pf, #1 of=#2, node distance=#3, label=#4:{#5}](#7) {};
  \node[pf, #1 of=#2, node distance=#3, xshift=1mm, yshift=-1mm](#8) {};
}

\title{Answering Hindsight Queries with Lifted Dynamic Junction Trees}
\author{Marcel Gehrke, Tanya Braun, and Ralf Möller\\ 
Institute of Information Systems, University of Lübeck, Lübeck  \\
\{gehrke, braun, moeller\}@ifis.uni-luebeck.de}
 
\begin{document}
\acrodef{SHR}[SHR]{Standard Health Record}
\acrodef{ehr}[EHR]{electronic health record}
\acrodef{FHIM}[FHIM]{Federal Health Information Model}
\acrodef{OHDSI}[OHDSI]{OMOP Common Data Model, the Observational Health Data Sciences and Informatics}
\acrodef{FHIR}[FHIR]{HL7’s FAST Healthcare Interoperability Resources}

\acrodef{jtree}[jtree]{junction tree}
\acrodef{plms}[PLMs]{probabilistic logical models}

\acrodef{pdb}[PDB]{probabilistic database}

\acrodef{pf}[parfactor]{parametric factor}
\acrodef{lv}[logvar]{logical variable}
\acrodef{rv}[randvar]{random variable}
\acrodef{crv}[CRV]{counting randvar}
\acrodef{prv}[PRV]{parameterised randvar}

\acrodef{fodt}[FO dtree]{first-order decomposition tree}

\acrodef{fojt}[FO jtree]{first-order junction tree}
\acrodef{ljt}[LJT]{lifted junction tree algorithm}
\acrodef{ldjt}[LDJT]{lifted dynamic junction tree algorithm}
\acrodef{lve}[LVE]{lifted variable elimination}

\acrodef{2tpm}[2TPM]{two-slice temporal parameterised model}
\acrodef{2tbn}[2TBN]{two-slice temporal bayesian network}

\acrodef{pm}[PM]{parameterised probabilistic model}
\acrodef{pdecm}[PDecM]{parameterised probabilistic decision model}

\acrodef{pdm}[PDM]{parameterised probabilistic dynamic model}
\acrodef{pddecm}[PDDecM]{parameterised probabilistic dynamic decision model}

\acrodef{dbn}[DBN]{dynamic Bayesian network}
\acrodef{bn}[BN]{Bayesian network}

\acrodef{dfg}[DFG]{dynamic factor graph}

\acrodef{dmln}[DMLN]{dynamic Markov logic network}

\acrodef{rdbn}[RDBN]{relational dynamic Bayesian network}

\acrodef{meu}[MEU]{maximum expected utility}
\acrodef{mldn}[MLDN]{Markov logic decision network}

\maketitle
\begin{abstract}
    The \ac{ldjt} efficiently answers \emph{filtering} and \emph{prediction} queries for probabilistic relational temporal models by building and then reusing a first-order cluster representation of a knowledge base for multiple queries and time steps.
    We extend \ac{ldjt} to
    \begin{enumerate*}
        \item solve the \emph{smoothing} inference problem to answer hindsight queries by introducing an efficient backward pass and
        \item discuss different options to instantiate a first-order cluster representation during a backward pass. 
    \end{enumerate*}   
    Further, our relational forward backward algorithm makes hindsight queries to the very beginning feasible.
    \ac{ldjt} answers multiple temporal queries faster than the static \acl{ljt} on an unrolled model, which performs \emph{smoothing} during message passing.   
        
\end{abstract}
\acresetall	

%

\section{Introduction}\label{sec:intro}

Areas like healthcare or logistics and cross-sectional aspects such as IT security involve probabilistic data with relational and temporal aspects and need efficient exact inference algorithms, as indicated by~\citeauthor{vlasselaer2014efficient} \shortcite{vlasselaer2014efficient}.
These areas involve many objects in relation to each other with changes over time and uncertainties about object existence, attribute value assignments, or relations between objects.  
More specifically, IT security involves network dependencies (relational) for many components (objects), streams of attacks over time (temporal), and uncertainties due to, for example, missing or incomplete information, caused by faulty sensor data.
By performing model counting, \acp{pdb} can answer queries for relational temporal models with uncertainties~\cite{dignos2012temporal,dylla2013temporal}.
However, each query embeds a process behaviour, resulting in huge queries with possibly redundant information.
In contrast to \acp{pdb}, we build more expressive and compact models including behaviour (offline) enabling efficient answering of more compact queries (online).
For query answering, our approach performs deductive reasoning by computing marginal distributions at discrete time steps.
In this paper, we study the problem of exact inference, in form of \emph{smoothing}, in large temporal probabilistic models. 

We introduce \acp{pdm} to represent probabilistic relational temporal behaviour and propose the \ac{ldjt} to exactly answer multiple \emph{filtering} and \emph{prediction} queries for multiple time steps efficiently \cite{gehrke2018ldjt}.
\ac{ldjt} combines the advantages of the interface algorithm~\cite{Murphy:2002:DBN} and the \ac{ljt}~\cite{BrMoe16a}.
Specifically, this paper extends \ac{ldjt} and contributes
\begin{enumerate*}
    \item an \emph{inter} \ac{fojt} backward pass to perform \emph{smoothing} for hindsight queries,
    \item different \ac{fojt} instantiation options during a backward pass, and
    \item a relational forward backward algorithm.
\end{enumerate*}
Our relational forward backward algorithm reinstantiates \acp{fojt} by leveraging \ac{ldjt}'s forward pass.
Without reinstantiating \acp{fojt}, the memory consumption of keeping all \acp{fojt} instantiated renders hindsight queries to the very beginning infeasible. 

Even though \emph{smoothing} is a main inference problem, to the best of our knowledge there is no approach solving \emph{smoothing} efficiently for relational temporal models.
Smoothing can improve the accuracy of hindsight queries by back-propagating newly gained evidence. 
Additionally, a backward pass is required for problems such as learning.

Lifting exploits symmetries in models to reduce the number of instances to perform inference on.
\ac{ljt} reuses the \ac{fojt} structure to answer multiple queries.
\ac{ldjt} also reuses the \ac{fojt} structure to answer queries for all time steps $t > 0$.
Additionally, \ac{ldjt} ensures a minimal exact \emph{inter} \ac{fojt} information propagation. 
Thus, \ac{ldjt} propagates minimal information to connect \acp{fojt} by message passing also during backward passes and reuses \ac{fojt} structures to perform \emph{smoothing}.

In the following, we begin by introducing \acp{pdm} as a representation for relational temporal probabilistic models and present \ac{ldjt}, an efficient reasoning algorithm for \acp{pdm}.
Afterwards, we extend \ac{ldjt} with an \emph{inter} \acp{fojt} backward pass and discuss different options to instantiate an \ac{fojt} during a backward pass. 
Lastly, we evaluate \ac{ldjt} against \ac{ljt} and 
 conclude by looking at extensions.

\section{Related Work}\label{sec:rel}

We take a look at inference for propositional temporal models, relational static models, and give an overview about research regarding relational temporal models.

For exact inference on propositional temporal models, a naive approach is to unroll the temporal model for a given number of time steps and use any exact inference algorithm for static, i.e., non-temporal, models.
In the worst case, once the number of time steps changes, one has to unroll the model and infer again.
\citeauthor{Murphy:2002:DBN} \shortcite{Murphy:2002:DBN} proposes the interface algorithm consisting of a forward and backward pass that uses a temporal d-separation with a minimal set of nodes to apply static inference algorithms to the dynamic model.

First-order probabilistic inference leverages the relational aspect of a static model.
For models with known domain size, it exploits symmetries in a model by combining instances to reason with representatives, known as lifting \cite{poole2003first}. 
\citeauthor{poole2003first} \shortcite{poole2003first} introduces parametric factor graphs as relational models and proposes \ac{lve} as an exact inference algorithm on relational models.
Further, \citeauthor{Braz07} \shortcite{Braz07}, \citeauthor{milch2008lifted} \shortcite{milch2008lifted}, and \citeauthor{TagFiDaBl13} \shortcite{TagFiDaBl13} extend \ac{lve} to its current form.
\citeauthor{lauritzen1988local} \shortcite{lauritzen1988local}  introduce the junction tree algorithm.
To benefit from the ideas of the junction tree algorithm and \ac{lve}, \citeauthor{BrMoe16a} \shortcite{BrMoe16a} present \ac{ljt} that efficiently performs exact first-order probabilistic inference on relational models given a set of queries.

Inference on relational temporal models mostly consists of approximative approaches.
Additionally, to being approximative, these approaches involve unnecessary groundings or are only designed to handle single queries efficiently.
\citeauthor{ahmadi2013exploiting} \shortcite{ahmadi2013exploiting} propose lifted (loopy) belief propagation.
From a factor graph, they build a compressed factor graph and apply lifted belief propagation with the idea of the factored frontier algorithm \cite{murphy2001factored}, which is an approximate counterpart to the interface algorithm and also provides means for a backward pass. 
\citeauthor{thon2011stochastic} \shortcite{thon2011stochastic} introduce CPT-L, a probabilistic model for sequences of relational state descriptions with a partially lifted inference algorithm.
\citeauthor{geier2011approximate} \shortcite{geier2011approximate} present an online interface algorithm for \acp{dmln}, similar to the work of  \citeauthor{papai2012slice} \shortcite{papai2012slice}.
Both approaches slice \acp{dmln} to run well-studied static MLN \cite{richardson2006markov} inference algorithms on each slice.  
Two ways of performing online inference using particle \emph{filtering} are described in \cite{manfredotti2009modeling,nitti2013particle}.

\citeauthor{vlasselaer2016tp} \shortcite{vlasselaer2016tp} introduce an exact approach for relational temporal models involving computing probabilities of each possible interface assignment. 

To the best of our knowledge, none of the relational temporal approaches perform \emph{smoothing} efficiently.
Besides of \ac{ldjt}'s benefits of being an exact algorithm answering multiple filter and \emph{prediction} queries for relation temporal models efficiently, we decided to extend \ac{ldjt} as it offers the ability to reinstatiate previous time steps and thereby make hindsight queries to the very beginning feasible. 

\section{Parameterised Probabilistic Models}\label{sec:back}
Based on \cite{BraMo17}, we shortly present \acp{pm} for relational static models.
Afterwards, we extend \acp{pm} to the temporal case, resulting in \acp{pdm} for relational temporal models, which, in turn, are based on \cite{gehrke2018ldjt}. 

\subsection{Parameterised Probabilistic Models}\label{pm}

\ac{pm}s combine first-order logic with probabilistic models, representing first-order constructs using \acp{lv} as parameters.
As an example, we set up a \ac{pm} for risk analysis with an attack graph (AG).
An AG models attacks on targeted components in a network.
A binary \ac{rv} holds if a component is compromised, which provides an attacker with privileges to further compromise a network to reach a final target.
We use \acp{lv} to represent users with certain privileges.
The model is inspired by \citeauthor{munoz2017exact} \shortcite{munoz2017exact}, who examine exact probabilistic inference for IT security with AGs.

\begin{definition}
    Let $\mathbf{L}$ be a set of \ac{lv} names, $\Phi$ a set of factor names, and $\mathbf{R}$ a set of \ac{rv} names. 
    A \ac{prv} $A = P(X^1,...,X^n)$ represents a set of \acp{rv} behaving identically by combining a \ac{rv} $P \in \mathbf{R}$ with $X^1,...,X^n \in \mathbf{L}$.
    If $n = 0$, the \ac{prv} is parameterless.
    The domain of a \ac{lv} $L$ is denoted by $\mathcal{D}(L)$.
    The term $range(A)$ provides possible values of a \ac{prv} $A$.
    Constraint $(\mathbf{X},C_\mathbf{X})$ allows to restrict \acp{lv} to certain domain values and is a tuple with a sequence of \acp{lv} $\mathbf{X} = (X^1,...,X^n)$ and a set $C_\mathbf{X} \subseteq \times_{i=1}^n \mathcal{D}(X^i)$.
    The symbol $\top$ denotes that no restrictions apply and may be omitted.
    The term $lv(Y)$ refers to the \acp{lv} 
    in some element $Y$.
    The term $gr(Y)$ denotes the set of instances of $Y$ with all \acp{lv} in $Y$ grounded w.r.t.\ constraints.
\end{definition}

From $\mathbf{R} = \{Server, User\}$ and $\mathbf{L} = \{X, Y\}$ with $\mathcal{D}(X) = \{x_1, x_2, x_3\}$ and $\mathcal{D}(Y) = \{y_1, y_2\}$, we build the boolean PRVs $Server$ and $User(X)$.
With $C = (X, \{x_1, x_2\})$, $gr(User(X)|C) = \{User(x_1), \linebreak User(x_2)\}$.
$gr(User(X)|\top)$ also contains $User(x_3)$.

\begin{figure}[b]
\centering
\scalebox{1}{
\begin{tikzpicture}[rv/.style={draw, ellipse, inner sep=1pt},pf/.style={draw, rectangle, fill=gray},label distance=0.2mm]
	\node[obs]					 								(S)	{$Server$};
    \pfs{above}{S}{7mm}{0}{$g^3$}{USa}{US}{USb}    
	\node[rv, left of=US, node distance=15mm, inner sep=0.5pt]			(U)	{$User(X)$};    
    \pfs{left}{U}{15mm}{330}{$g^0$}{U1a}{U1}{U1b}    
	\node[rv, left of=U1, node distance=15mm]						(T1)	{$Attack1$};    
    \pfs{below}{S}{7mm}{0}{$g^4$}{ASa}{AS}{ASb}    
	\node[rv, left of=AS, node distance=15mm, inner sep=0.5pt]			(A)	{$Admin(Y)$};    
    \pfs{left}{A}{15mm}{30}{$g^1$}{A1a}{A1}{A1b}    
	\node[rv, left of=A1, node distance=15mm]						(T2)	{$Attack2$};    
    \pfs{below left}{U}{10mm}{0}{$g^2$}{UAa}{UA}{UAb}
	\node[rv, left of=UA, node distance=20mm, inner sep=0.5pt]			(FW)	{$Infects(X,Y)$};
    
	\draw (U) -- (US);
	\draw (US) -- (S);
    
	\draw (U) -- (U1);
	\draw (U1) -- (T1);
    
	\draw (A) -- (AS);
	\draw (AS) -- (S);
    
	\draw (A) -- (A1);
	\draw (A1) -- (T2);
    
	\draw (U) -- (UA);
	\draw (UA) -- (A);
	\draw (UA) -- (FW);
\end{tikzpicture}
}
\caption{Parfactor graph for $G^{ex}$ }
\label{fig:swe}
\end{figure}
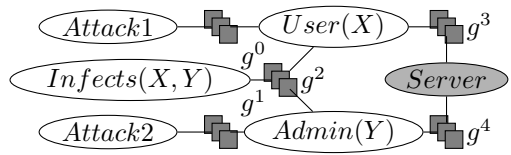

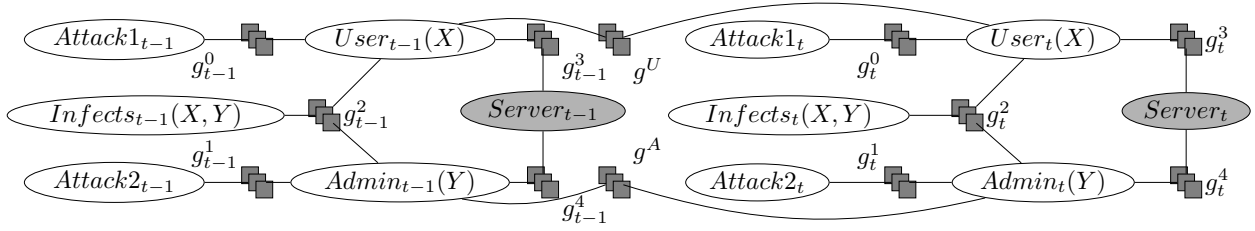
\begin{figure*}[t]
\center
\scalebox{0.95}{
\begin{tikzpicture}[rv/.style={draw, ellipse, inner sep=1pt},pf/.style={draw, rectangle, fill=gray},label distance=0.2mm]
	\node[obs]					 								(S)	{$Server_{t-1}$};
    \pfs{above}{S}{10mm}{335}{$g^3_{t-1}$}{USa}{US}{USb}    
	\node[rv, left of=US, node distance=20mm, inner sep=0.5pt]			(U)	{$User_{t-1}(X)$};    
    \pfs{left}{U}{20mm}{200}{$g^0_{t-1}$}{U1a}{U1}{U1b}    
	\node[rv, left of=U1, node distance=20mm]						(T1)	{$Attack1_{t-1}$};    
    \pfs{below}{S}{10mm}{335}{$g^4_{t-1}$}{ASa}{AS}{ASb}    
	\node[rv, left of=AS, node distance=20mm, inner sep=0.5pt]			(A)	{$Admin_{t-1}(Y)$};    
    \pfs{left}{A}{20mm}{160}{$g^1_{t-1}$}{A1a}{A1}{A1b}    
	\node[rv, left of=A1, node distance=20mm]						(T2)	{$Attack2_{t-1}$};    
    \pfs{below left}{U}{15mm}{0}{$g^2_{t-1}$}{UAa}{UA}{UAb}
	\node[rv, left of=UA, node distance=25mm, inner sep=0.5pt]			(FW)	{$Infects_{t-1}(X,Y)$};
    
	\node[obs, right of=S, node distance = 9cm]					 		(S2)	{$Server_t$};
    \pfs{above}{S2}{10mm}{0}{$g^3_t$}{USa}{US2}{USb}    
	\node[rv, left of=US2, node distance=20mm, inner sep=0.5pt]			(U2)	{$User_t(X)$};    
    \pfs{left}{U2}{20mm}{200}{$g^0_t$}{U1a}{U12}{U1b}    
	\node[rv, left of=U12, node distance=20mm]						(T12)	{$Attack1_t$};    
    \pfs{below}{S2}{10mm}{0}{$g^4_t$}{ASa}{AS2}{ASb}    
	\node[rv, left of=AS2, node distance=20mm, inner sep=0.5pt]			(A2)	{$Admin_t(Y)$};    
    \pfs{left}{A2}{20mm}{160}{$g^1_t$}{A1a}{A12}{A1b}    
	\node[rv, left of=A12, node distance=20mm]						(T22)	{$Attack2_t$};    
    \pfs{below left}{U2}{15mm}{0}{$g^2_t$}{UAa}{UA2}{UAb}
	\node[rv, left of=UA2, node distance=25mm, inner sep=0.5pt]			(FW2)	{$Infects_t(X,Y)$};
    
    \node[right of = S, node distance = 1cm] (I) {};
    
    \pfs{below}{I}{10mm}{45}{$g^A$}{UAa}{IA}{UAb}
    \pfs{above}{I}{10mm}{315}{$g^U$}{UAa}{IU}{UAb}

    \path [-, bend left=15] (IA) edge node {} (A);
    \path [-, bend right=15] (IA) edge node {} (A2);   
    
    \path [-, bend right=15] (IU) edge node {} (U);
    \path [-, bend left=15] (IU) edge node {} (U2);
    
	\draw (U) -- (US);
	\draw (US) -- (S);   
	\draw (U) -- (U1);
	\draw (U1) -- (T1);  
	\draw (A) -- (AS);
	\draw (AS) -- (S); 
	\draw (A) -- (A1);
	\draw (A1) -- (T2);
	\draw (U) -- (UA);
	\draw (UA) -- (A);
	\draw (UA) -- (FW);

	\draw (U2) -- (US2);
	\draw (US2) -- (S2);  
	\draw (U2) -- (U12);
	\draw (U12) -- (T12);
	\draw (A2) -- (AS2);
	\draw (AS2) -- (S2);
	\draw (A2) -- (A12);
	\draw (A12) -- (T22);
	\draw (U2) -- (UA2);
	\draw (UA2) -- (A2);
	\draw (UA2) -- (FW2);


\end{tikzpicture}
}
\caption{$G_\rightarrow^{ex}$ the two-slice temporal parfactor graph for model $G^{ex}$}
\label{fig:TSPG}	
\end{figure*}

\begin{definition}
We denote a \ac{pf} $g$ with
$\forall \mathbf{X} : \phi(\mathcal{A})\;| C$,
$\mathbf{X} \subseteq \mathbf{L}$ being a set of \acp{lv} over which the factor generalises, $C$ a constraint on $\mathbf{X}$, and $\mathcal{A} = (A^1,...,A^n)$ a sequence of \acp{prv}.
We omit $(\forall \mathbf{X} :)$ if $\mathbf{X} = lv(\mathcal{A})$.
A function $\phi : \times_{i=1}^n range(A^i) \mapsto \mathbb{R}^+$ with name $\phi \in \Phi$ is identical for all grounded instances of $\mathcal{A}$.
The complete specification for $\phi$ is a list of all input-output values.
A \ac{pm} $G := \{g^i\}_{i=0}^{n-1}$ is a set of \acp{pf} and semantically represents the full joint probability distribution $P_G = \frac{1}{Z} \prod_{f \in gr(G)} \phi(f)$ with $Z$ as normalisation constant.
\end{definition}

Adding boolean PRVs $Attack1$, $Attack2$, $Admin(Y)$, $Infects(X,Y)$, $G_{ex} = \{g^i\}^4_{i=0}$,
 $g^0 = \phi^0(Attack1, User(X))$, 
 $g^1 = \phi^1(Attack2, Admin(Y))$, 
 $g^2 = \phi^2(User(X), Admin(Y), Infects(X,Y))$, 
 $g^3 = \phi^3(Server, User(X))$, and
 $g^4 = \phi^4(Server, Admin(Y))$
forms a model.
$g^2$ has eight, the others four input-output pairs (omitted). 
Constraints are $\top$, i.e., the $\phi$'s hold for all domain values.
E.g., $gr(g^0)$ contains three factors with identical $\phi$.
\Cref{fig:swe} depicts $G^{ex}$ as a graph with six variable nodes for the PRVs and five factor nodes for $g^0$ to $g^4$ with edges to the PRVs involved.
Additionally, we can observe the state of the server.
The remaining \acp{prv} are latent.

The semantics of a model is given by grounding and building a full joint distribution.
In general, queries ask for a probability distribution of a \ac{rv} using a model's full joint distribution and given fixed events as evidence. 

\begin{definition}
    Given a \ac{pm} $G$, a ground \ac{prv} $Q$ and grounded \ac{prv}s with fixed range values $\mathbf{E}$, the expression $P(Q|\mathbf{E})$ denotes a query w.r.t.\ $P_G$.
\end{definition}


\subsection{Parameterised Probabilistic Dynamic Models}\label{sec:pdm}

To define \acp{pdm}, we use \acp{pm} and the idea of how \acp{bn} give rise to \acp{dbn}. 
We define \acp{pdm} based on the first-order Markov assumption, i.e., a time slice $t$ only depends on the previous time slice $t-1$. 
Further, the underlining process is stationary, i.e., the model behaviour does not change over time. 

\begin{definition}
    A \ac{pdm} is a pair of \acp{pm} $(G_0,G_\rightarrow)$ where
        $G_0$ is a PM representing the first time step and  
        $G_\rightarrow$ is a \acl{2tpm} representing $\mathbf{A}_{t-1}$ and $\mathbf{A}_t$ where
    $\mathbf{A}_\pi$ a set of \acp{prv} from time slice $\pi$.
\end{definition}

\Cref{fig:TSPG} shows 
$G_\rightarrow^{ex}$ consisting of $G^{ex}$ for time step $t-1$ and $t$ with \emph{inter}-slice \acp{pf} for the behaviour over time. 
In this example, the \acp{pf} $g^{A}$ and $g^U$ are the \emph{inter}-slice \acp{pf}, modelling the temporal behavior. 

\begin{definition}
    Given a \ac{pdm} $G$, a ground \ac{prv} $Q_t$ and grounded \ac{prv}s with fixed range values $\mathbf{E}_{0:t}$ the expression $P(Q_t|\mathbf{E}_{0:t})$ denotes a query w.r.t.\ $P_G$.
\end{definition}

The problem of answering a marginal distribution query $P(A^i_\pi|\mathbf{E}_{0:t})$ w.r.t.\ the model is called \emph{prediction} for $\pi > t$, \emph{filtering} for $\pi = t$, and \emph{smoothing} for $\pi < t$.
\section{Lifted Dynamic Junction Tree Algorithm}\label{sec:fodjt}
We start by introducing \ac{ljt}, mainly based on \cite{braun2017preventing}, to provide means to answer queries for \acp{pm}.
Afterwards, we present \ac{ldjt}, based on \cite{gehrke2018ldjt}, consisting of \ac{fojt} constructions for a \ac{pdm} and an efficient \emph{filtering} and \emph{prediction} algorithm. 

\subsection{Lifted Junction Tree Algorithm}\label{ljt}

\ac{ljt} provides efficient means to answer queries $P(\mathbf{Q}|\mathbf{E})$, with a set of query terms, given a \ac{pm} $G$ and evidence $\mathbf{E}$, by performing the following steps:
\begin{enumerate*}
    \item Construct an \ac{fojt} $J$ for $G$.
    \item Enter $\mathbf{E}$ in $J$.
    \item Pass messages.
    \item Compute answer for each query $Q^i \in \mathbf{Q}$.
\end{enumerate*}

We first define an \ac{fojt} and then go through each step. 
To define an \ac{fojt}, we first need to define parameterised clusters (parclusters), the nodes of an \ac{fojt}.

\begin{definition}
    A parcluster $\mathbf{C}$ is defined by $\forall \mathbf{L} : \mathbf{A} | C$.
    $\mathbf{L}$ is a set of \ac{lv}s, $\mathbf{A}$ is a set of \ac{prv}s with $lv(\mathbf{A}) \subseteq \mathbf{L}$, and $C$ a constraint on $\mathbf{L}$.
    We omit $(\forall \mathbf{L} :)$ if $\mathbf{L} = lv(\mathbf{A})$.
    A parcluster $\mathbf{C}^i$ can have parfactors $\phi(\mathcal{A}^\phi) | C^\phi $ assigned given that
    \begin{enumerate*}
        \item $\mathcal{A}^\phi \subseteq \mathbf{A}$,
        \item $lv(\mathcal{A}^\phi) \subseteq \mathbf{L}$, and
        \item $C^\phi \subseteq C$
    \end{enumerate*}
    holds.
    We call the set of assigned \ac{pf}s a local model $G^i$.\\
An \ac{fojt} for a \ac{pm} $G$ is $J=(\mathbf{V},\mathbf{E})$ where $J$ is a cycle-free graph,
the nodes $\mathbf{V}$ denote a set of parcluster, and the set $\mathbf{E}$ edges between parclusters. 
An \ac{fojt} must satisfy the properties:
\begin{enumerate*}
    \item A parcluster $\mathbf{C}^i$ is a set of \acp{prv} from $G$.
    \item For each \ac{pf} $\phi(\mathcal{A}) | C$  in G, $\mathcal{A}$ must appear in some parcluster $\mathbf{C}^i$.
    \item If a \ac{prv} from $G$ appears in two parclusters $\mathbf{C}^i$ and $\mathbf{C}^j$, it must also appear in every parcluster $\mathbf{C}^k$ on the path connecting nodes i and j in $J$.
\end{enumerate*}
The separator $\mathbf{S}^{ij}$ containing shared \acp{prv} of edge $i-j$ is given by $\mathbf{C}^i \cap \mathbf{C}^j$. 
\end{definition}

\ac{ljt} constructs an \ac{fojt} using a \acl{fodt}, enters evidence in the \ac{fojt}, and passes messages through an \emph{inbound} and an \emph{outbound} pass, to distribute local information of the nodes through the \ac{fojt}.
To compute a message, \ac{ljt} eliminates all non-seperator \acp{prv} from the parcluster's local model and received messages.
After message passing, \ac{ljt} answers queries.
For each query, LJT finds a parcluster containing the query term and sums out all non-query terms in its local model and received messages.

\begin{figure}[b]
    \centering
    \scalebox{0.7}{
\begin{tikzpicture}[pc/.style={draw, rounded corners=6pt, align=center, node distance=46mm, inner sep=2pt}]
	\node[pc, label={[gray, inner sep=1pt]270:{\scriptsize$\{g^0\}$}},label={[font=\scriptsize]90:{$\mathbf{C}^1$}}]				(c1) {$Attack1,$\\$ User(X)$};
	\node[pc, right of=c1, label={[gray, inner sep=1pt]270:{\scriptsize$\{g^2, g^3, g^4\}$}},label={90:{\scriptsize$\mathbf{C}^2$}}]	(c2) {$Server, User(X),$\\$ Admin(Y), Infects(X,Y)$};
	\node[pc, right of=c2, label={[gray, inner sep=1pt]270:{\scriptsize$\{g^1\}$}},label={90:{\scriptsize$\mathbf{C}^3$}}]			(c3) {$Attack2,$\\$ Admin(Y)$};
	\draw (c1) -- node[inner sep=1pt, pin={[yshift=-2.4mm]90:{\scriptsize$\{User(X)\}$}}]		{} (c2);
	\draw (c2) -- node[inner sep=1pt, pin={[yshift=-2.4mm]90:{\scriptsize$\{Admin(Y)\}$}}]	{} (c3);
\end{tikzpicture}
}
\caption{FO jtree for $G^{ex}$ (local models in grey)}
\label{fig:fojt}
\end{figure}
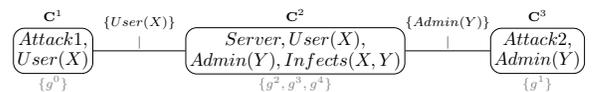

\Cref{fig:fojt} shows an \ac{fojt} of $G^{ex}$ with the local models of the parclusters and the separators as labels of edges.
During the \emph{inbound} phase of message passing, \ac{ljt} sends messages from $\mathbf{C}^1$ and $\mathbf{C}^3$ to $\mathbf{C}^2$ and during the \emph{outbound} phase from $\mathbf{C}^2$ to $\mathbf{C}^1$ and $\mathbf{C}^3$.
If we want to know whether $Attack1$ holds, we query for $P(Attack1)$ for which \ac{ljt} can use parcluster $\mathbf{C}^1$.
\ac{ljt} sums out $User(X)$ from $\mathbf{C}^1$'s local model $G^{1}$, $\{g^0\}$, combined with the received messages, here, one message from $\mathbf{C}^2$. 

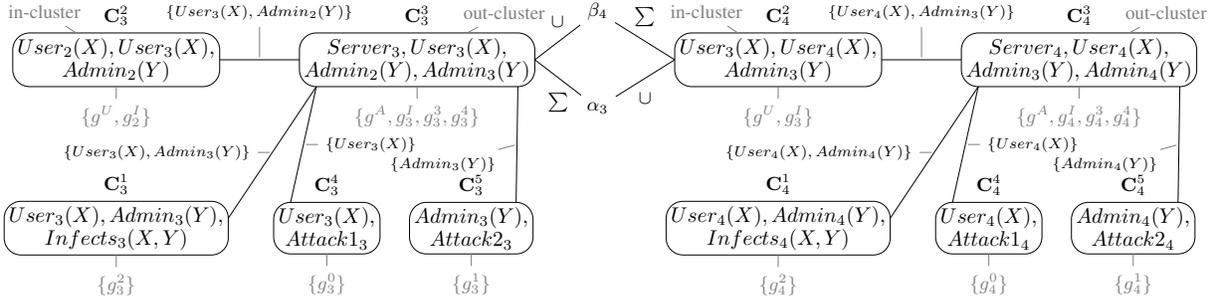
\begin{figure*}[t]
\center
\scalebox{0.8}{
\begin{tikzpicture}[every node/.style={font=\footnotesize}, node distance=40mm]
    \node[pc2, pin={[pin distance=2mm, gray, align=center, text width=3cm]270:
    {$\{g^U, g^I_{2}\}$}},
    pin={[pin distance=1.5mm, gray, align=center]130:{in-cluster}},
    label={90:{$\mathbf{C}_3^2$}}](c2) 
    {$ User_{2}(X),  User_3(X),$\\$ Admin_{2}(Y)$};
    \node[pc2, node distance=28mm, below of = c2, pin={[pin distance=2mm, gray]270:
    {$\{g^2_3 \}$}},
    label={90:{$\mathbf{C}_3^1$}}](c1) 
    {$User_3(X), Admin_3(Y),$\\$ Infects_3(X,Y)$};
    \node[pc2, node distance=50mm, right of = c2, pin={[pin distance=2mm, gray]270:
    {$\{g^{A}, g^I_3, g^3_3,g^4_3 \}$}},
    pin={[pin distance=1.5mm, gray, align=center]40:{out-cluster}},
    label={90:{$\mathbf{C}_3^3$}}](c3) 
    {$Server_3, User_3(X),$\\$Admin_{2}(Y), Admin_3(Y)$};
    \node[pc2, node distance=35mm, right of = c1, pin={[pin distance=2mm, gray]270:
    {$\{g^0_3 \}$}},
    label={90:{$\mathbf{C}_3^4$}}](c4) 
    {$User_3(X),$\\$ Attack1_3$};
    \node[pc2, node distance=24mm, right of = c4, pin={[pin distance=2mm, gray]270:
    {$\{g^1_3 \}$}},
    label={90:{$\mathbf{C}_3^5$}}](c5) 
    {$Admin_3(Y),$\\$ Attack2_3$};
    
    \node[right of = c3, node distance=3cm] (c) {};
    \node[below of = c, node distance=0.8cm] (c11) {$\alpha_3$};
    \node[above of = c, node distance=0.8cm] (b) {$\beta_4$};

    \node[pc2, right of = c, node distance=3cm, pin={[pin distance=2mm, gray, align=center, text width=3cm]270:
    {$\{g^U, g^I_{3}\}$}},
    pin={[pin distance=1.5mm, gray, align=center]130:{in-cluster}},
    label={90:{$\mathbf{C}_4^2$}}](c7) 
    {$ User_{3}(X),  User_4(X),$\\$ Admin_{3}(Y)$};
    \node[pc2, node distance=28mm, below of = c7, pin={[pin distance=2mm, gray]270:
    {$\{g^2_4 \}$}},
    label={90:{$\mathbf{C}_4^1$}}](c6) 
    {$User_4(X), Admin_4(Y),$\\$ Infects_4(X,Y)$};
    \node[pc2, node distance=50mm, right of = c7, pin={[pin distance=2mm, gray]270:
    {$\{g^{A}, g^I_4, g^3_4,g^4_4 \}$}},
    pin={[pin distance=1.5mm, gray, align=center]40:{out-cluster}},
    label={90:{$\mathbf{C}_4^3$}}](c8) 
    {$Server_4, User_4(X),$\\$Admin_{3}(Y), Admin_4(Y)$};
    \node[pc2, node distance=35mm, right of = c6, pin={[pin distance=2mm, gray]270:
    {$\{g^0_4 \}$}},
    label={90:{$\mathbf{C}_4^4$}}](c9) 
    {$User_4(X),$\\$ Attack1_4$};
    \node[pc2, node distance=24mm, right of = c9, pin={[pin distance=2mm, gray]270:
    {$\{g^1_4 \}$}},
    label={90:{$\mathbf{C}_4^5$}}](c10) 
    {$Admin_4(Y),$\\$ Attack2_4$};
    
\draw (c1.before north east) -- node[inner sep=1pt, pin={[pin distance=2mm, align=center]180:\scriptsize{$\{User_3(X), Admin_3(Y)\}$}}] {} (c3.after south west);
\draw (c3) -- node[inner sep=1pt, pin={[pin distance=5mm, align=center]90:\scriptsize{$\{User_3(X), Admin_{2}(Y)\}$}}] {} (c2);
\draw (c3.after south west) -- node[inner sep=1pt, pin={[pin distance=2mm, align=center]0:\scriptsize{$\{ User_3(X)\}$}}] {} (c4.before north west);
\draw (c3.before south east) -- node[ inner sep=1pt, pin={[pin distance=2mm, align=center]200:\scriptsize{$\{ Admin_3(Y)\}$}}] {} (c5.after north east);

\draw (c6.before north east) -- node[inner sep=1pt, pin={[pin distance=2mm, align=center]180:\scriptsize{$\{User_4(X), Admin_4(Y)\}$}}] {} (c8.after south west);
\draw (c8) -- node[inner sep=1pt, pin={[pin distance=5mm, align=center]90:\scriptsize{$\{User_4(X), Admin_{3}(Y)\}$}}] {} (c7);
\draw (c8.after south west) -- node[inner sep=1pt, pin={[pin distance=2mm, align=center]0:\scriptsize{$\{ User_4(X)\}$}}] {} (c9.before north west);
\draw (c8.before south east) -- node[ inner sep=1pt, pin={[pin distance=2mm, align=center]200:\scriptsize{$\{ Admin_4(Y)\}$}}] {} (c10.after north east);

    \path [-] (c11) edge node [label= below:{$\sum$}] {} (c3.east);
    \path [-] (c11) edge node [label= below:{$\cup$}] {} (c7.west);
    \path [-] (b) edge node [label= above:{$\cup$}] {} (c3.east);
    \path [-] (b) edge node [label= above:{$\sum$}] {} (c7.west);

\end{tikzpicture}
}
\caption{Forward and backward pass of \ac{ldjt} (local models and in- and out-cluster labeling in grey)}
\label{fig:fojt1}	
\end{figure*}

\subsection{LDJT: Overview} 

\ac{ldjt} efficiently answers queries $P(\mathbf{Q}_t|\mathbf{E}_{0:t})$, with a set of query terms $\{\mathbf{Q}_t\}_{t=0}^T$, given a \ac{pdm} $G$ and evidence $\{\mathbf{E}_t\}_{t=0}^T$, by performing the following steps:

\begin{enumerate*}[label=(\roman*)]
    \item Offline construction of two \acp{fojt} $J_0$ and $J_t$ with \emph{in-} and \emph{out-clusters} from $G$
    \item For $t=0$, using $J_0$ to enter $\mathbf{E}_0$, pass messages, answer each query term $Q_\pi^i \in \mathbf{Q}_0$, and preserve the state in message $\alpha_0$
    \item For $t>0$, instantiate $J_t$ for the current time step $t$, recover the previous state from message $\alpha_{t-1}$, enter $\mathbf{E}_t$ in $J_t$, pass messages, answer each query term $Q_\pi^i \in \mathbf{Q}_t$, and preserve the state in message $\alpha_t$
\end{enumerate*}

We begin with \ac{ldjt}'s \acp{fojt} construction, which contain a minimal set of \acp{prv} to m-separate the \acp{fojt}.
M-separation means that information about these \acp{prv} make \acp{fojt} independent from each other.
Afterwards, we present how \ac{ldjt} connects \acp{fojt} for reasoning to solve the \emph{filtering} and \emph{prediction} problems efficiently.

\subsection{LDJT: FO Jtree Construction for PDMs}\label{ldjt:const} 

\ac{ldjt} constructs \acp{fojt} for $G_0$ and $G_\rightarrow$, both with an incoming and outgoing interface.
To be able to construct the interfaces in the \acp{fojt}, \ac{ldjt} uses the \ac{pdm} $G$ to identify the interface \acp{prv} $\mathbf{I}_t$ for a time slice $t$.
\begin{definition}
    The forward interface is defined as $\mathbf{I}_{t} = \{A_{t}^i \mid \exists \phi(\mathcal{A}) | C \in G :  A_{t}^i \in \mathcal{A} \wedge \exists A_{t+1}^j \in \mathcal{A}\}$, i.e., the \acp{prv} which have successors in the next slice. 
\end{definition}

\ac{prv}s $User_{t-1}(X)$ and $Admin_{t-1}(Y)$ from $G_{\rightarrow}^{ex}$, shown in \cref{fig:TSPG}, have successors in the next time slice, making up $\mathbf{I}_{t-1}$.
To ensure interface \acp{prv} $\mathbf{I}$ ending up in a single parcluster, \ac{ldjt} adds a \ac{pf} $g^I$ over the interface to the model.
Thus, \ac{ldjt} adds a \ac{pf} $g^I_0$ over $\mathbf{I}_0$ to $G_0$, builds an \ac{fojt} $J_0$, and labels the parcluster with $g^I_0$ from $J_0$ as \emph{in-} and \emph{out-cluster}.
For $G_\rightarrow$, \ac{ldjt} removes all non-interface \acp{prv} from time slice $t-1$, adds \acp{pf} $g^I_{t-1}$ and $g^I_{t}$, constructs $J_t$, and labels 
the parcluster containing $g^I_{t-1}$ as \emph{in-cluster} and the parcluster containing $g^I_{t}$ as \emph{out-cluster}.

The interface \acp{prv} are a minimal required set to m-separate the \acp{fojt}.
\ac{ldjt} uses these \acp{prv} as separator to connect the \emph{out-cluster} of $J_{t-1}$ with the \emph{in-cluster} of $J_t$, allowing to reuse the structure of $J_t$ for all $t>0$.

\subsection{LDJT: Reasoning with PDMs}\label{sec:forward}

Since $J_0$ and $J_t$ are static, \ac{ldjt} uses \ac{ljt} as a subroutine by passing on a constructed \ac{fojt}, queries, and evidence for step $t$ to handle evidence entering, message passing, and query answering using the \ac{fojt}.
Further, for proceeding to the next time step, \ac{ldjt} calculates an $\alpha_t$ message over the interface \acp{prv} using the \emph{out-cluster} to preserve the information about the current state.
Afterwards, \ac{ldjt} increases $t$ by one, instantiates $J_t$, and adds $\alpha_{t-1}$ to the \emph{in-cluster} of $J_t$.
During message passing, $\alpha_{t-1}$ is distributed through $J_t$.
Thereby, \ac{ldjt} performs an \emph{inter} \ac{fojt} forward pass to proceed in time.
Additionally, due to the \emph{inbound} and \emph{outbound} message passing, \ac{ldjt} also performs an \emph{intra} backward pass for the current \ac{fojt}. 

\Cref{fig:fojt1} depicts the passing on of the current state from time step three to four.
To capture the state at $t=3$, \ac{ldjt} sums out the non-interface \acp{prv} $Server_3$ and $Admin_2(Y)$ from $\mathbf{C}_3^3$'s local model and the received messages and saves the result in message $\alpha_3$.
After increasing $t$ by one, \ac{ldjt} adds $\alpha_3$ to the \emph{in-cluster} of $J_4$, $\mathbf{C}_4^2$.
$\alpha_3$ is then distributed by message passing and accounted for during calculating $\alpha_4$.

\section{Smoothing Extension for LDJT}\label{sec:smooth}
We introduce an \emph{inter} \ac{fojt} backward pass and extend \ac{ldjt} with it to also answer \emph{smoothing} queries efficiently.

\subsection{Inter FO Jtree Backward Pass}
Using the forward pass, each instantiated \ac{fojt} contains evidence from the initial time step up to the current time step.
The \emph{inter} \acp{fojt} backward pass propagates information to previous time steps, allowing \ac{ldjt} to answer marginal distribution queries $P(A^i_\pi|\mathbf{E}_{0:t})$ with $\pi < t$. 

The backward pass, similar to the forward pass, uses the interface connection of the \acp{fojt}, calculating a message over the interface \acp{prv} for an \emph{inter} \ac{fojt} message pass.
To perform a backward pass, \ac{ldjt} uses the \emph{in-cluster} of the current \ac{fojt} $J_t$ to calculate a $\beta_t$ message over the interface \acp{prv}.
\ac{ldjt} first has to remove the $\alpha_{t-1}$ message from the  \emph{in-cluster} of $J_t$, since $J_t$ received the $\alpha_{t-1}$ message from the destination of the $\beta_t$ message.
After \ac{ldjt} calculates $\beta_t$ by summing out all non-interface \acp{prv}, it decreases $t$ by one. 
Finally, \ac{ldjt} instantiates the \ac{fojt} for the new time step and adds the $\beta_{t+1}$ message to the \emph{out-cluster} of $J_t$.

\Cref{fig:fojt1} also depicts how \ac{ldjt} performs a backward pass.
\ac{ldjt} uses the \emph{in-cluster} of $J_4$ to calculate $\beta_4$ by summing out all non-interface \acp{prv} of $\mathbf{C}_4^2$'s local model without $\alpha_3$.
After decreasing $t$ by one, \ac{ldjt} adds $\beta_4$ to the \emph{out-cluster} of $J_3$.
$\beta_4$ is then distributed and accounted for in $\beta_3$.

The forward and backward pass instantiate \acp{fojt} from the corresponding structure given a time step.
However, since \ac{ldjt} already instantiates \acp{fojt} during a forward pass, it has different options to instantiate \acp{fojt} during a backward pass.
The first option is to keep all instantiated \acp{fojt} from the forward pass and the second option is to reinstantiate \acp{fojt} using evidence and $\alpha$ messages.
The second option, to reinstantiate previous time steps is only possible by leveraging how \ac{ldjt}'s forward pass is defined.

\subsubsection{Preserving FO Jtree Instantiations}
To keep all instantiated \acp{fojt}, including computed messages, is quite time-efficient since the option reuses already performed computation.
Thereby, during an \emph{intra} \ac{fojt} message pass, \ac{ldjt} only needs to account for the $\beta$ message.
By selecting the \emph{out-cluster} as the root node for message passing, this leads to $n-1$ instead of $2 \times (n-1)$ messages, where $n$ is the number of parclusters.
The required \ac{fojt} is already instantiated and does not need to be instantiated. 
The main drawback is the memory consumption.
Each \ac{fojt} contains all computed messages, evidence, and structure.

\subsubsection{FO Jtree Reinstantiation}
Leveraging \ac{ldjt}'s forward pass, another approach is to reinstantiate \acp{fojt} on demand during a backward pass using evidence and $\alpha$ messages.
\ac{ldjt} repeats the steps to instantiate the \ac{fojt} for which it only needs to save the $\alpha$ message and the evidence.
Thus, \ac{ldjt} can reinstantiate \acp{fojt} on-demand.

The main drawback are the repeated computations.
After \ac{ldjt} instantiates the \ac{fojt}, it enters evidence, $\alpha$, and $\beta$ messages to perform a complete message pass.
Thereby, \ac{ldjt} repeats computations compared to keeping the instantiations and calculating messages can be costly, as in the worst case the problem is exponential to the number of \acp{prv} to be eliminated \cite{taghipour2013first}.

In case \ac{ldjt} only reinstantiates an \ac{fojt} to calculate a $\beta_t$ message, meaning there are no \emph{smoothing} queries for that time step, \ac{ldjt} can calculate the $\beta_t$ message with only $n-1$ messages.
By selecting the \emph{in-cluster} as root, \ac{ldjt} has already after $n-1$ messages (\emph{inbound} pass) all required information in the \emph{in-cluster} to calculate a $\beta_t$ message. 

\subsection{Extended LDJT}
\cref{alg:LDJT} shows the general steps of \ac{ldjt} including the backward pass, as an extension to the original \ac{ldjt} \cite{gehrke2018ldjt}.
\ac{ldjt} uses the function \emph{DFO-JTREE} to construct the structures of the \acp{fojt} $J_0$ and $J_t$ and the set of interface \acp{prv} as described in \cref{ldjt:const}.
Afterwards, \ac{ldjt} enters a while loop, which it leaves after reaching the last time step, and performs the routine of entering evidence, message passing and query answering for the current time step.
Lastly, \ac{ldjt} performs one forward pass, as described in \cref{sec:forward}, to proceed in time.

\begin{algorithm}[t]
    \caption{LDJT Alg. for \ac{pdm} $(G_0,G_\rightarrow)$, Queries $\{\mathbf{Q}\}_{t=0}^T$, Evidence $\{\mathbf{E}\}_{t=0}^T$}
    \label{alg:LDJT}
    \begin{algorithmic}
        \Procedure{LDJT}{$G_0,G_\rightarrow, \{\mathbf{Q}\}_{t=0}^T, \{\mathbf{E}\}_{t=0}^T$}
            \State $t := 0$
            \State $(J_0,J_t, \mathbf{I}_t) : =$  DFO-JTREE$(G_0,G_\rightarrow)$ 
            \While{$t \neq T+1$}
            \State $J_{t} :=$ LJT.EnterEvidence($J_{t}, \mathbf{E}_t$)
            \State $J_{t} :=$ LJT.PassMessages($J_{t}$)
            \For{$q_{\pi} \in \mathbf{Q}_t$}
            \State AnswerQuery($J_0, J_t, q_{\pi}, \mathbf{I}_t, \alpha, t$)
            \EndFor
            \State $(J_t,t, \alpha[t-1]) :=$ ForwardPass$(J_0, J_t, t, \mathbf{I}_t)$
        \EndWhile
        \EndProcedure
    \end{algorithmic}  
    \hrulefill
    \begin{algorithmic}
        \Procedure{AnswerQuery}{$J_0, J_t, q_{\pi}, \mathbf{I}_t, \alpha, t$} 
        \While{$t \neq \pi$}
            \If{$t>\pi$}
                \State $(J_t, t) :=$ BackwardPass$(J_0, J_t,  \mathbf{I}_t, \alpha[t-1], t)$ 
            \Else
                \State $(J_t, t, \_) :=$ ForwardPass$(J_0, J_t, \mathbf{I}_t, t)$                           
            \EndIf
        \State LJT.PassMessages$(J_t)$
        \EndWhile
        \State \textbf{print} LJT.AnswerQuery$(J_t, q_{\pi})$
        \EndProcedure
    \end{algorithmic}
    \hrulefill
    \begin{algorithmic}
        \Function{ForwardPass}{$J_0, J_t, \mathbf{I}_t, t$}
            \State $\alpha_t := \sum_{J_t(\text{out-cluster}) \setminus \mathbf{I}_t} J_t(\text{out-cluster})$ 
            \State $t := t+1$
            \State $J_t(\text{in-cluster}) := \alpha_{t-1} \cup J_t(\text{in-cluster})$
            \State\Return $(J_t, t, \alpha_{t-1})$      
        \EndFunction
    \end{algorithmic}
    \hrulefill
    \begin{algorithmic}
        \Function{BackwardPass}{$J_0, J_t, \mathbf{I}_t, \alpha_{t-1}, t$}
            \State $\beta_t := \sum_{J_t(\text{in-cluster}) \setminus \mathbf{I}_t} (J_t(\text{in-cluster}) \setminus \alpha_{t-1})$ 
            \State $t := t-1$
            \State $J_t(\text{out-cluster}) := \beta_{t+1} \cup J_t(\text{out-cluster})$
            \State\Return $(J_t, t)$      
        \EndFunction
    \end{algorithmic}
\end{algorithm}

The main extension of \cref{alg:LDJT} is in the query answering function.
First, \ac{ldjt} identifies the type of query, namely \emph{filtering}, \emph{prediction}, and \emph{smoothing}.
To perform \emph{filtering}, \ac{ldjt} passes the query and the current \ac{fojt} to \ac{ljt} to answer the query.
\ac{ldjt} applies the forward pass until the time step of the query is reached to answer the query for \emph{prediction} queries.
To answer \emph{smoothing} queries, \ac{ldjt} applies the backward pass until the time step of the query is reached and answers the query.
Further, \ac{ldjt} uses \ac{ljt} for message passing to account for $\alpha$ respectively $\beta$ messages.

Let us now illustrate how \ac{ldjt} answers \emph{smoothing} queries.
We assume that the server is compromised at time step $1983$ and we want to know whether $Admin(y_1)$ infected $User(x_1)$ at time step $1973$ and whether $User(x_1)$ is compromised at timestep $1978$.
Hence, \ac{ldjt} answers the marginal distribution queries $P(\mathbf{Q}_{1983}|\mathbf{E}_{0:1983})$, where the new evidence $\mathbf{E}_{1983}$ consists of $\{Server_{1983} = true\}$ and the set of query terms $\mathbf{Q}_{1983}$ consists of at least the query terms $\{User_{1978}(x_1), Infects_{1973}(x_1,y_1)\}$.

\ac{ldjt} enters the evidence $\{Server_{1983} = true\}$ in $J_{1983}$ and passes messages.
To answer the queries, \ac{ldjt} performs a backward pass and first calculates $\beta_{1983}$ by summing out $User_{1983}(X)$ from $\mathbf{C}^2_{1983}$'s local model and received messages without $\alpha_{1982}$.
\ac{ldjt} adds the $\beta_{1983}$ message to $\mathbf{C}^3_{1982}$'s local model and passes messages in $J_{1982}$ using \ac{ljt}.
In such a manner \ac{ldjt} proceeds until it reaches time step $1978$ and thus propagated the information to $J_{1978}$. 

Having $J_{1978}$, \ac{ldjt} can answer the marginal distribution query $P(User_{1978}(x_1)|\mathbf{E}_{0:1983})$.
To answer the query, \ac{ljt} can sum out $Attack1_{1978}$ and $User_{1978}(X)$ where $X \neq x_1$ from $\mathbf{C}^4_{1978}$'s local model and the received message from $\mathbf{C}^3_{1978}$. 
To answer the other marginal distribution query $P(Infects_{1973}(x_1,y_1)|\mathbf{E}_{0:1983})$, \ac{ldjt} performs additional backward passes until it reaches time step $1973$ and then uses again \ac{ljt} to answer the query term $Infects_{1973}(x_1,y_1)$ given $J_{1973}$.
Even though, \cref{alg:LDJT} states that \ac{ldjt} has to start the \emph{smoothing} query for time step $1973$ from $1983$, \ac{ldjt} can reuse the computations it performed for the first \emph{smoothing} query.

\begin{theorem}
     \ac{ldjt} is correct regarding \emph{smoothing}.
\end{theorem}

\begin{proof}
	Each \ac{fojt} contains evidence up to the time step the \ac{fojt} is instantiated for.
	To perform \emph{smoothing}, \ac{ldjt} distributes information, including evidence, from the current \ac{fojt} $J_t$ backwards.
    Therefore, \ac{ldjt} performs an \emph{inter} \acp{fojt} backward message pass over the interface separator.
    The $\beta_t$  message is correct, since calculating the $\beta_t$ message, the \emph{in-cluster} received all messages from its neighbours and removes the $\alpha_{t-1}$ message, which originated from the designated receiver. 
    The $\beta_t$ message, which \ac{ldjt} adds to the \emph{out-cluster} of $J_{t-1}$, is then accounted for during the message pass inside $J_{t-1}$ and thus, during the calculation of $\beta_{t-1}$.
    Following this approach, every \ac{fojt} included in the backward pass contains all evidence.
    Thus, it suffices to apply the backward pass until \ac{ldjt} reaches the desired time step and does not need to apply the backward pass until $t=0$.
    Hence, \ac{ldjt} propagates back information until it reaches the desired time step and performs \emph{filtering} on the corresponding \ac{fojt} to answer the query.
\end{proof}

\subsection{Discussion}

As \ac{ldjt} has two options to instantiate previous time steps, we discuss how to combine the options efficiently.
Further, \ac{ldjt} can also leverage calculations from the current time step for multiple \emph{smoothing} and \emph{prediction} queries. 

\subsubsection{Combining Instantiation Options}
To provide the queries to \ac{ldjt}, a likely scenario is a predefined set of queries for each time step with an option for additional queries on-demand.
In such a scenario, reoccurring \emph{smoothing} queries are in the predefined set, also called fixed-lag \emph{smoothing}, and therefore, the number of time steps for which \ac{ldjt} has to perform a backward pass is known.
With a known fixed-lag a combination of our two options is advantageous.
Assuming the largest \emph{smoothing} lag is $10$, \ac{ldjt} can always keep the last $10$ \acp{fojt} in memory and reinstantiate \acp{fojt} on-demand.
Further, in case an on-demand \emph{smoothing} query has a lag of $20$, \ac{ldjt} can reinstantiate the \acp{fojt} starting with $J_{t-11}$.
Thereby, \ac{ldjt} can keep a certain number of \acp{fojt} instantiated, for fast query answering, and in case a \emph{smoothing} query is even further in the past, reconstruct the \acp{fojt} on demand using evidence and $\alpha$ messages. 
Hence, combining the approaches is a good compromise between time and space efficiency.

To process a data stream, one possibility is to reason over a sliding window that proceeds with time \cite{oezMoeNeu15}.
Thereby, each window stores a processable amount of data.
To keep \acp{fojt} instantiated to faster answer hindsight queries is comparable to sliding windows in stream data processing.
\ac{ldjt} keeps only a reasonable amount of \acp{fojt} and the window slides to the next time step, when \ac{ldjt} proceeds in time.

\subsubsection{Reusing Computations for One Time Step}
During query answering for one time step, \ac{ldjt} can also reuse computations. 
For example, let us assume, we have two \emph{smoothing} queries, one with a lag of $2$ and the other with a lag of $4$.
\ac{ldjt} can reuse the calculations it performed during the \emph{smoothing} query with a lag of $2$, namely it can start the backward pass for the query with lag $4$ at $J_{t-2}$ and does not need to recompute the already performed two backward passes.
To reuse the computations, there are two options.

The first option is that the \emph{smoothing} queries are sorted based on the time difference to the current time step.
Here, \ac{ldjt} can keep the \ac{fojt} from the last \emph{smoothing} query and perform additional backward passes, but does not repeat computations that lead to the \ac{fojt} of the last \emph{smoothing} query.
The second option is similar to the \ac{fojt} reinstantition, namely to keep the calculated $\beta$ messages for the current time step and reinstantiate the \ac{fojt} closest to the currently queried time step.
Analogously, \ac{ldjt} can reuse computations for answering \emph{prediction} queries. 

Under the presence of \emph{prediction} queries, \ac{ldjt} does not have to recompute $\alpha_t$ after it answered all queries, since \ac{ldjt} already computed $\alpha_t$ during a \emph{prediction} query with the very same evidence in the \ac{fojt} present.
Unfortunately, given new evidence for a new time step all other $\alpha$ and $\beta$ messages that \ac{ldjt} calculated during \emph{prediction} and \emph{smoothing} queries from the previous time step are invalid. 

\section{Evaluation}\label{sec:eval}

We compare \ac{ldjt} against \ac{ljt} provided with the unrolled model for multiple maximum time steps.
To be more precise, we compare, for each maximum time step, the runtime of \ac{ldjt} with instant complete \ac{fojt} reinstantiation, the worst case, against the runtime of \ac{ljt} directly provided all information for all time steps, resulting in one message pass, which is the best case for \ac{ljt}.
For the evaluation, we use the example model $G^{ex}$ with the set of evidence being empty. 

We start by defining a set of queries that is executed for each time step and then evaluate the runtimes of \ac{ldjt} and \ac{ljt}.
Our predefined set of queries for each time step is: $\{Server_t$, $User_t(x_1)$, $Admin_t(y_1)\}$ with lag $0$, $2$, $5$, and $10$.
For each time step these $12$ queries are executed.

\begin{figure}[t]
    \includegraphics[width=0.45\textwidth]{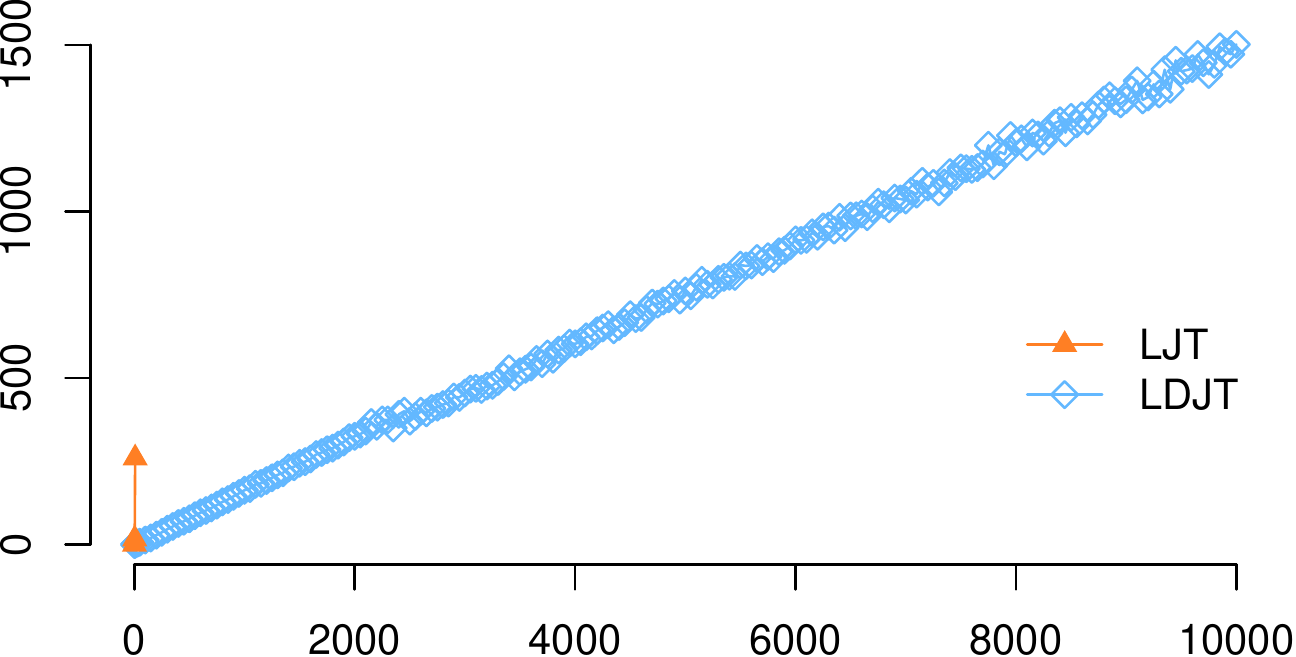}
  \caption{Runtimes [seconds], x-axis: maximum time steps}
  \label{fig:times_overall}	
\end{figure}

Now, we evaluate how long  \ac{ldjt} and \ac{ljt} take to answer the queries for up to $10000$ time steps.
\Cref{fig:times_overall} shows the runtime in seconds for each maximum time step.
We can see that the  runtime of \ac{ldjt} (diamond) to answer the questions is linear to the maximum number of time steps. 
Thus, \ac{ldjt} more or less needs a constant time to answer queries once it instantiated the corresponding \ac{fojt} and also the time to perform a forward or backward pass is more or less constant w.r.t.\ time, no matter how far \ac{ldjt} already proceeded in time.
For \ac{ldjt} the runtimes for each operation are independent from the current time step, since the structure of the model stays the same over time.
To be more precise, for our example \ac{ldjt} takes about ${\sim}5$ms to initially construct the two \acp{fojt}.
For a forward or backward pass, \ac{ldjt} roughly needs ${\sim}6.5$ms.
Both passes include a complete message pass, which roughly takes ${\sim}6$ms. 
Thus, most of the time for a forward or backward pass is spent on message passing.
To answer a query, \ac{ldjt} needs on average ${\sim}5$ms.
To obtain the runtimes, we used a virtual machine having a 4 core Intel Xeon E5 with 2.4 GHz , 16 GB of RAM, and running Ubuntu 14.04.5 LTS (64 Bit).

Providing the unrolled model to \ac{ljt}, it produces results for the first $8$ time steps with a reduced set of queries.
Here, we can see that the runtime of \ac{ljt} appears to be exponential to the maximum number of time steps, which is expected.
The \ac{fojt} construction of \ac{ljt} is not optimised for the temporal case, such as creating an \ac{fojt} similar to an unrolled version of \ac{ldjt}'s \ac{fojt}.
Therefore, the number of \acp{prv} in a parcluster increases with additional time steps in \ac{ljt}.
With additional time steps, the unrolled model becomes larger and while constructing an \ac{fojt}, \acp{prv} that depend on each other are more likely to be clustered in a parcluster.
Thus, with more \acp{prv}, the number of \acp{prv} in a parcluster is expected to grow.
In our example, maximum number of \acp{prv} in a parcluster for $4$ time steps is $14$ and for $8$ time steps is $27$.
For a ground jtree, the complexity of variable elimination is exponential to the largest cluster in the jtree \cite{darwiche2009modeling}.
Thus, we can explain why \ac{ljt} is more or less exponential to the maximum number of time steps.

\begin{figure}[t]
    \includegraphics[width=0.45\textwidth]{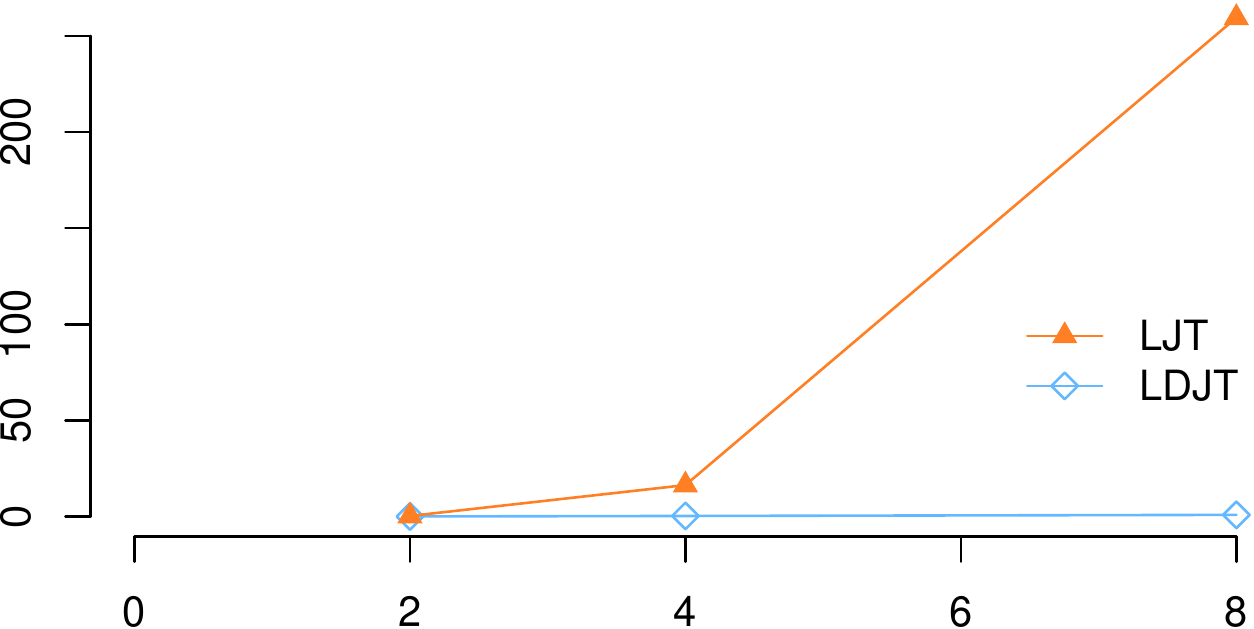}
  \caption{Runtimes [seconds], x-axis: maximum time steps}
  \label{fig:times_ljt}	
\end{figure}


Overall, \cref{fig:times_overall} shows
\begin{enumerate*}
     \item how crucial proper handling of temporal aspects is and 
     \item that \ac{ldjt} performs really well.
\end{enumerate*}
Further, we can see that \ac{ldjt} can handle combinations of different query types, such as \emph{smoothing} and \emph{filtering}.
With evidence as input for all time steps the runtime is only marginally slower.
Executing the example with evidence for all time steps produces an overhead of roughly ${\sim}8$ms on average for each time step.
Therefore, for each time step, reinstantiating $10$ \acp{fojt} with evidence and performing message passing only produces an overhead of ${\sim}8$ms compared to having an empty set of evidence. 
The linear behaviour with increasing maximum time steps is the expected and desired behaviour for an algorithm handling temporal aspects.
Furthermore, from a runtime complexity of \ac{ldjt} there should be no difference in either performing a \emph{smoothing} query with lag $10$ or a \emph{prediction} query with $10$ time steps into the future.
Thus, the previous results that \ac{ldjt} outperforms the ground case still hold.

For the evaluation, the more time efficient version would be to always keep the last $10$ \acp{fojt}.
Each message pass takes ${\sim}6$ms on average and during each backward pass \ac{ldjt} performs a message pass.
By keeping the instantiated \acp{fojt}, \ac{ldjt} can halve the number of messages during a message passing.
Assuming that the runtime of the message pass is linear to the number of calculated messages, resulting in reducing the runtime of each message pass to ${\sim}3$ms by keeping the \acp{fojt}.
Further, for each time step, \ac{ldjt} performs $10$ backward passes each with a message pass.
Thus, \ac{ldjt} can reduce the runtime by ${\sim}30$ms per time step and an overall reduce the runtime by ${\sim}300$s, which is about $20\%$ of the overall runtime for $10000$ timesteps.

\section{Conclusion}\label{sec:conc}

We present a complete \emph{inter} \ac{fojt} backward pass for \ac{ldjt}, allowing to perform \emph{smoothing} efficiently for relational temporal models.
\ac{ldjt} answers multiple queries efficiently by reusing a compact \ac{fojt} structure for multiple queries.
Due to temporal m-separation, which is ensured by the \emph{in-} and \emph{out-clusters}, \ac{ldjt} uses the same compact structure for all time steps $t > 0$.
Thus, \ac{ldjt} also reuses the structure during a backward pass.
Further, it reuses computations from a forward pass during a backward pass.
\ac{ldjt}'s relational forward backward pass also makes \emph{smoothing} queries to the very beginning feasible.
First results show that \ac{ldjt} significantly outperforms \ac{ljt}.
 
We currently work on extending \ac{ldjt} to also calculate the most probable explanation as well as a solution to the maximum expected utility problem. 
The presented backward pass could also be helpful to deal with incrementally changing models.
Additionally, it would be possible to reinstantiate an \ac{fojt} for a backward pass solely given the $\alpha$ messages.
Other interesting future work includes a tailored automatic learning for \acp{pdm}, parallelisation of \ac{ljt}. 

\subsubsection*{Acknowledgement}
This research originated from the Big Data project being part of Joint Lab 1, funded by Cisco Systems, at the centre COPICOH, University of Lübeck
    
\bibliographystyle{aaai}
\bibliography{../bib/tex/bib}

\end{document}